\documentclass[10pt, conference, letterpaper]{IEEEtran}
\usepackage{times}  
\usepackage{helvet} 
\usepackage{courier}  
\usepackage[hyphens]{url}  
\usepackage{graphicx} 
\usepackage{amsmath}
\usepackage{amssymb}
\usepackage{subfigure}
\usepackage{tikz}
\usepackage{pgfplotstable}

\usepackage{graphicx}  
\usepackage[ruled,vlined]{algorithm2e}
\def\bea{\begin{eqnarray}}
\def\eea{\end{eqnarray}}
\def\nn{\nonumber}
\newtheorem{theorem}{Theorem}
\newtheorem{lemma}{Lemma}

\newtheorem{proof}{Proof}

\title{Too Much Information Kills Information: \\ A Clustering Perspective}
\author{Yicheng Xu\textsuperscript{\rm 1} Vincent Chau\textsuperscript{\rm 1}\thanks{Corresponding author.} Chenchen Wu\textsuperscript{\rm 2} Yong Zhang\textsuperscript{\rm 1} Vassilis Zissimopoulos\textsuperscript{\rm 3} Yifei Zou\textsuperscript{\rm 4}\\
\textsuperscript{\rm 1}Shenzhen Institutes of Advanced Technology, Chinese Academy of Sciences, P.R.China.\\ \{yc.xu, vincentchau, zhangyong\}@siat.ac.cn\\
\textsuperscript{\rm 2}Tianjin University of Technology, P.R.China. wu\_chenchen\_tjut@163.com\\
\textsuperscript{\rm 3} National and Kapodistrian University of Athens, Greece. vassilis@di.uoa.gr\\
\textsuperscript{\rm 4}The University of Hong Kong, P.R.China. yfzou@cs.hku.hk
}
 
\begin{document}

\maketitle

\begin{abstract}
Clustering is one of the most fundamental tools in the artificial intelligence area, particularly in the pattern recognition and learning theory.
In this paper, we propose a simple, but novel approach for variance-based $k$-clustering tasks, included in which is the widely known $k$-means clustering.
The proposed approach picks a sampling subset from the given dataset and makes decisions based on the data information in the subset only.
With certain assumptions, the resulting clustering is provably good to estimate the optimum of the variance-based objective with high probability. Extensive experiments
on synthetic datasets and real-world datasets show that to obtain competitive results compared with $k$-means method (Llyod 1982) and $k$-means++ method (Arthur and Vassilvitskii 2007),
we only need 7\% information of the dataset. If we have up to 15\% information of the dataset, then our algorithm outperforms both the $k$-means method and $k$-means++ method
in at least 80\% of the clustering tasks, in terms of the quality of clustering. Also, an extended algorithm based on the same idea guarantees
a balanced $k$-clustering result.
\end{abstract}

\section{Introduction}
Cluster analysis is a subarea of machine learning that studies methods of unsupervised discovery of homogeneous subsets of data instances from heterogeneous datasets.
Methods of cluster analysis have been successfully applied in a wide spectrum of areas of image processing, information retrieval, text mining and cybersecurity.
Cluster analysis has a rich history in disciplines such as biology, psychology, archaeology, psychiatry, geology and geography, even through
there is an increasing interest in the use of clustering methods in very hot fields like natural language processing, recommended system, image and video processing, etc.
The importance and interdisciplinary nature of clustering is evident through its vast literature.

The goal of variance-based $k$-clustering is to find a $k$ sized partition of a given dataset so as to minimize the sum of the within-cluster variances. The well-known $k$-means is a variance-based
clustering which defines the within-cluster variance as the sum of squared distances from each data to the means of the cluster it belongs to. The folklore of $k$-means method \cite{l1982},
also known as the Lloyd's algorithm, is still one of the top ten popular data mining algorithms and is implemented as a standard clustering method in most
machine learning libraries, according to \cite{wu2008}. To overcome the high sensitivity to proper initialization, \cite{km++}
propose the $k$-means++ method by augmenting the $k$-means method with a
careful randomized seeding preprocessing. The $k$-means++ method is proved to be $O(\log k)$-competitive with the optimal clustering and the analysis is tight. Even through it is easy
to implement, $k$-means++ has to make a full pass through the dataset for every single pick of the seedings, which leads to a high complexity. \cite{km||} drastically reduce the number
 of passes needed to obtain, in parallel, a good initialization. The proposed $k$-means$\|$ obtains a nearly optimal solution after a logarithmic number of passes, and in
practice a constant number of passes suffices. Following this path, there are several speed-ups or hybrid methods. For example, \cite{b2016} replace the seeding method in $k$-means++
with a substantially faster approximation based on Markov Chain Monte Carlo sampling. The proposed method retains the full theoretical
guarantees of $k$-means++ while its computational complexity is only sublinear in the number of data points. A simple combination of $k$-means++ with a local search strategy  achieves a
constant approximation guarantee in expectation and is more competitive in practice \cite{betterkm}. Furthermore, the number of local search steps is dramatically reduced from $O(k\log \log k)$ to
$\epsilon k$ while maintaining the constant performance guarantee \cite{arxiv2020}.

A balanced clustering result is often required in a variety of applications. However, many existing clustering algorithms have good clustering performances, yet fail in producing balanced clusters.
The balanced clustering, which requires size constraints for the resulting clusters, is at least APX-hard in general under the assumption P$\neq$NP \cite{hardness}. It
 attracts research interests simultaneously from approximation and heuristic perspectives.
Heuristically, \cite{AAAI17} apply the method of augmented Lagrange multipliers to minimize the least square linear regression in order to regularize the clustering model. The proposed
approach not only produces good clustering performance but also guarantees a balanced clustering result.
To achieve more accurate clustering for large scale dataset,  exclusive lasso on $k$-means and min-cut are leveraged  to regulate the
balance degree of the clustering results. By optimizing the objective functions that build atop the exclusive lasso, one can
make the clustering result as much balanced as possible \cite{AAAI18}.
Recently, \cite{xmy} introduce a balance regularization term in the objective function of $k$-means and by replacing the assignment step of $k$-means method with a simplex algorithm
they give a fast algorithm for soft-balanced clustering, and the hard-balanced requirement can be satisfied by enlarging the multiplier in the regularization term. Also, there are some
algorithmic results for balanced $k$-clustering tasks with valid performance guarantees.
The first constant approximation algorithm for the variance based hard-balanced clustering is a $(69+\epsilon)$-approximation in fpt-time \cite{xu}.
The approximation ratio is then improved to $7+\epsilon$  \cite{ICALP} and $1+\epsilon$ \cite{QPTAS} sequentially with the same asymptotic running time.

\textbf{Our contributions}
In this paper, we propose a simple, but novel algorithm based on random sampling that computes provably good $k$-clustering results for variance based clustering tasks.
An extended version based on the same idea is valid for balanced $k$-clustering tasks with hard size constraints.  We make cross comparisons between the proposed Random Sampling method
with the $k$-means method and the $k$-means++ method in both synthetic datasets and real-world datasets. The numerical results show that our
method is competitive with the $k$-means method and $k$-means++ method with a sampling size of only 7\% of the dataset. When the sampling size reaches 15\% or higher,
the Random Sampling method outperforms both the $k$-means method and the $k$-means++ method in at least 80\% rounds of the clustering tasks.

The remainder of the paper is organized as follows. In the Warm-up section, we mainly provide some preliminaries towards a better understand of the proposed algorithm.
In the Random Sampling section, we present the main algorithm and the analysis. After that, we provide the performance of the proposed algorithm on different datasets
in the Numerical Results section. Then we extend the proposed algorithm to deal with the balanced clustering tasks in the Extension section. In the last section, we discuss the
advantages as well as disadvantages of the proposed algorithm, and some promising areas where our algorithm has the potential to outperform existing clustering methods.

\section{Warm-up}
\subsection{Variance-Based $k$-Clustering}
Roughly speaking, clustering tasks seek an organization of a collection of patterns into clusters based on similarity, such that patterns within a cluster are very similar
while patterns from different clusters are highly dissimilar. One way to measure the similarity is the so-called variance-based objective function, that leverages the squared distances
between patterns and the centroid of the cluster they belong to.

A well-known variance-based clustering task is the $k$-means clustering, which is a method of vector quantization that originally comes up from signal processing, which aims to partition
$n$ real vectors (quantification from colors) into $k$ clusters so as to minimize the within-cluster variances. What makes the $k$-means clustering different from other variance-based
$k$-clustering is the way it measures the similarity. The $k$-means defines the similarity between vectors as the squared Euclidean distance between them.
For simplicity, we mainly take the $k$-means as an example in the later discussion
but most of the results carry over to the general variance-based $k$-clustering tasks.

The $k$-means clustering can be formally described as follows. Given are a data set $X=\{x_1, x_2, \cdots, x_n \}$ and an integral number $k$, where each data in $X$ is a $d$-dimensional
real vector. The objective is to partition $X$ into $k(\le n)$ disjoint subsets so as to minimize the total within-cluster sum of squared distances (or variances).
For a fixed finite data set $A\subseteq  \mathbb{R}^d$, the centroid (also known as the means) of $A$ is denoted by
$c(A):= \sum_{x \in A} x/|A|.$
Therefore, the objective of the $k$-means clustering is to find a partition $\{X_1, X_2, \cdots, X_k\}$ of $X$ such that the following is minimized:
$$\sum\limits^k_{i=1}\sum\limits_{x\in X_i} ||x-c(X_i)||^2,$$
where $||a-b||$ denotes the Euclidean distance between vectors $a$ and $b$.

Also, we will extend our result to a general scenario of balanced clustering, where the capacity constraints must be satisfied.
For the balanced $k$-clustering, the only difference is additional global constraints for the size of the clusters. Both lower bound and upper bound constraints are considered in this paper.
Based on the above, the balanced $k$-means can be described as finding a partition $\{X_i\}_{1\le i \le k}$ of $X$ so as to minimize the aforementioned $k$-means objective and
$$l\le |X_i| \le u, ~{\rm for~all} ~~1\le i\le k.$$
Obviously, by taking appropriate values for $l$ and $u$, we reduce it to the $k$-means clustering. Thus,
it is more difficult to obtain an optimal balanced $k$-means clustering.

\subsection{Voronoi Diagram and Centroid Lemma}
Solving the optimal $k$-means clustering for an arbitrary data set is NP-hard. However, Lloyd proposes a fast local search based heuristic for $k$-means clustering, also known as
the $k$-means method. A survey of data mining techniques states that it is by far the most popular clustering algorithm used in scientific and industrial applications. The $k$-means method
is carried out through iterative Voronoi Diagram construction, combined with the centroid adjustment according to the Centroid Lemma.

Voronoi Diagram is a partition of a space into regions close to each of a given set of centers. Formally, given centers $C=\{c_1,c_2,...,c_k\}$ in $\mathbb{R}^d$ for example,
the Voronoi Diagram w.r.t. (with respect to) $C$ consists of the following Voronoi cells defined for $i=1,2,...,k$ as
$${\rm Cell}(i)=\{x\in\mathbb{R}^d: d(x,c_i)\le d(x,c_j) {\rm~for~all~} j\neq i\}.$$
\begin{figure}[h]
\centering
\subfigure{
\begin{minipage}{0.47\linewidth}
\centering
\includegraphics[height=1.0in,width=1.5in]{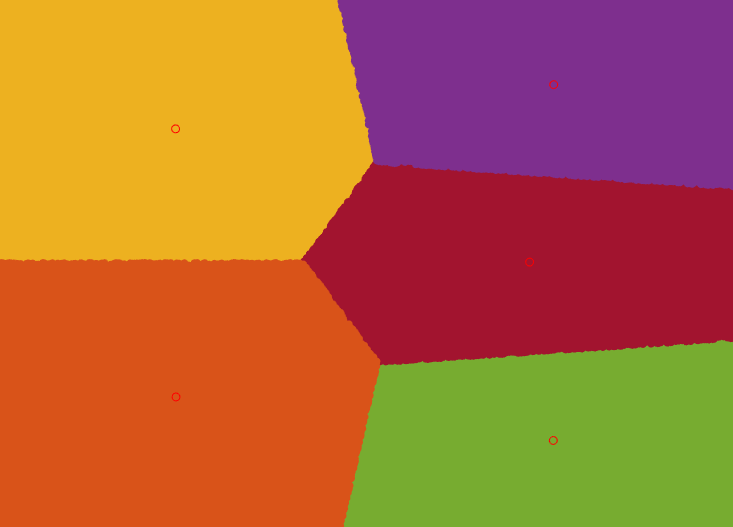}
\end{minipage}%
}%
\subfigure{
\begin{minipage}{0.48\linewidth}
\centering
\includegraphics[height=1.0in,width=1.5in]{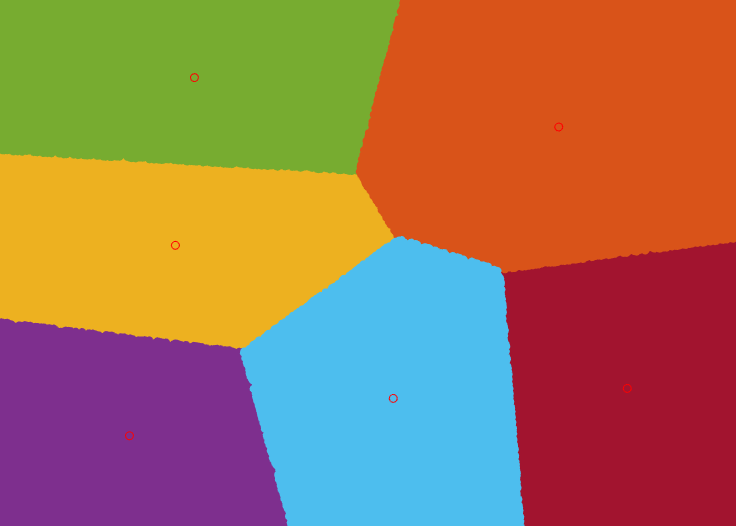}
\end{minipage}%
}%
\quad
\subfigure{
\begin{minipage}{0.48\linewidth}
\centering
\includegraphics[height=1.0in,width=1.5in]{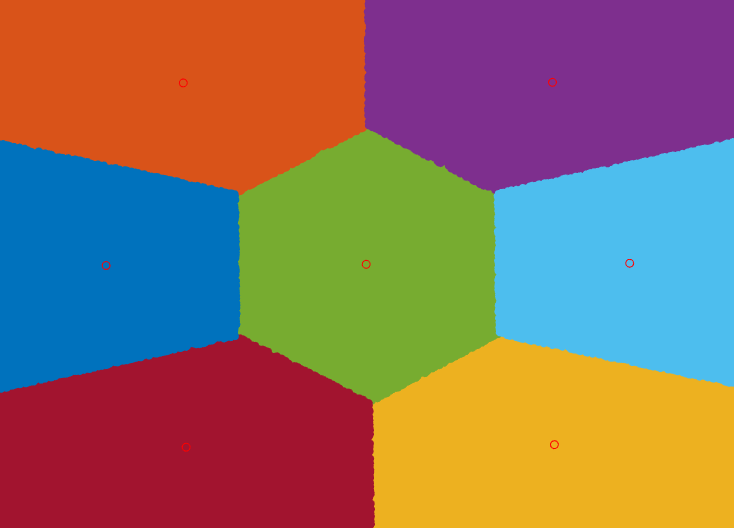}
\end{minipage}}%
\subfigure{
\begin{minipage}{0.48\linewidth}
\centering
\includegraphics[height=1.0in,width=1.5in]{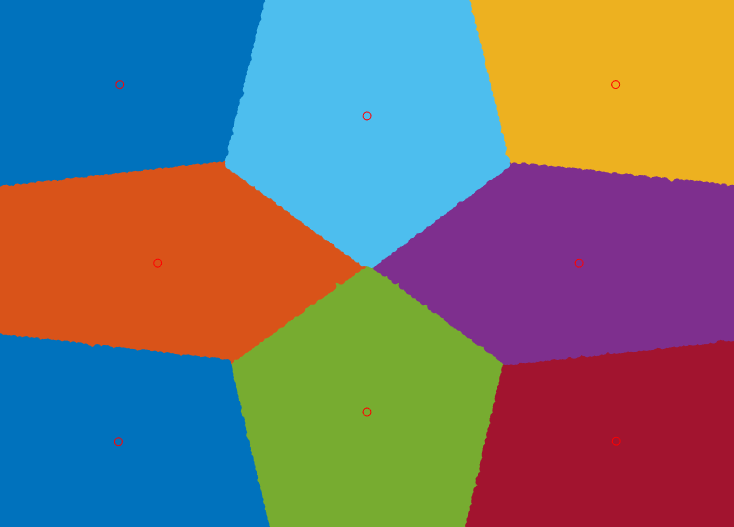}
\end{minipage}
}%
\centering
\caption{Examples of Voronoi diagram in the plane}
\label{vor}
\end{figure}
See Figure \ref{vor} as examples of the Voronoi Diagrams in the plane. Obviously, any Voronoi Diagram $\Pi$ of $\mathbb{R}^d$ gives a feasible partition for any set $X\subseteq \mathbb{R}^d$ (ties broken arbitrarily), which is called the Voronoi Partition of $X$
w.r.t. $\Pi$. More precisely, the Voronoi Partition of $X$ is given by $\{X_i\}_{1\le i\le k}$, where $X_i=X\cap{\rm Cell}(i)$.

On the other hand, given $X\subseteq \mathbb{R}^d$, it holds for any $v\in \mathbb{R}^d$ that
$$\sum\limits_{x\in X} ||x-v||^2=\sum\limits_{x\in X} ||x-c(X)||^2+|X|\cdot ||c(X)-v||^2,$$
which is the so-called Centroid Lemma. An example of application of the Centroid Lemma refers to \cite{jv}.
Note that the Centroid Lemma implies that the centroid/means of a cluster is the minimizer of the within-cluster variance.

\section{Random Sampling}
Given a dataset $X$, we say $S\subseteq X$ is a random sampling of $X$ if $S$ is obtained by several independent draws from $X$ uniformly at random.
We show that it is not bad to estimate the objective value of the variance-based $k$-clustering of $X$ using $S$.
Before that, we introduce two basic facts on expectation and variance from probability theory.
Given independent random variables $V_1$ and $V_2$, we have the follows.
\begin{description}
  \item[Fact 1] $E(aV_1+bV_2)=aE(V_1)+bE(V_2)$
  \item[Fact 2] $var(aV_1+bV_2)=a^2var(V_1)+b^2var(V_2)$
\end{description}

Suppose $S$ is an $m$-draws random sampling of $X$.
Then, $c(S)$ is an unbiased estimation of $c(X)$ and the squared Euclidean distance between them can be estimated by the following lemma.

\begin{lemma}\label{basic}
$E(c(S))=c(X)$,
$E(||c(S)-c(X)||^2)=\frac{1}{m}var(X)$.
\end{lemma}
\begin{proof}
Assume w.o.l.g. that $S=\{V_1,V_2,...,V_m\}$ and recall $V_i$ are independent random variables. Based on Fact 1, it holds that
\bea
E(c(S))&=&E(\frac{1}{m}\sum_{i=1}^mV_i)= \frac{1}{m}\sum_{i=1}^mE(V_i)\nn\\
&=&\frac{1}{m}\sum_{i=1}^mc(X)=c(X). \nn
\eea
Then
\bea
E(||c(S)-c(X)||^2)&=&E(||c(S)-E(c(S))||^2) \nn\\
&=& var(c(S))\nn\\
&=& var(\frac{1}{m}\sum_{i=1}^mV_i) \nn\\
&=& \frac{1}{m^2}\sum_{i=1}^m var(V_i) \nn\\
&=& \frac{1}{m}var(X), \nn
\eea
where the second last equality is derived from Fact 2.
\end{proof}

Based on the above, we conclude that $c(S)$ is indeed a good estimate for $c(X)$. A natural idea comes from here that
it is probably a good estimate for  $\sum_{x\in X}||x-c(X)||^2$ using $\sum_{x\in X}||x-c(S)||^2$, as given in the
following lemma.

\begin{lemma}\label{markov}
With probability at least $1-\delta$,
$$\sum_{x\in X}||x-c(S)||^2\le (1+\frac{1}{m\delta})\sum_{x\in X}||x-c(X)||^2.$$
\end{lemma}
\begin{proof}
From Lemma \ref{basic} and the Markov Inequality we know, with probability at least $1-\delta$,
$$||c(S)-c(X)||^2\le \frac{1}{m\delta}\sum_{x\in X}||x-c(X)||^2.$$
Recalling the Centroid Lemma, immediately we have with probability at least $1-\delta$ that
\bea
&&\sum_{x\in X}||x-c(S)||^2 \nn\\
 &=& \sum_{x\in X}||x-c(X)||^2+|X|\cdot||c(S)-c(X)||^2\nn\\
&\le& (1+\frac{1}{m\delta})\sum_{x\in X}||x-c(X)||^2, \nn
\eea
completing the proof.
\end{proof}

Consider the following randomized algorithm for the $k$-clustering task based on the random sampling idea, which we simply call Random Sampling.
Given the sampling set $S$, we construct every $k$-clustering of $S$ by a brute force search. Note that there are $O(m^{dk})$ many possibilities due to \cite{socg94}, but
we are allowed to do this because $S$ is much smaller than $X$.
For each $k$-clustering of $S$, we divide the $\mathbb{R}^d$ space into $k$ Voronoi cells according to the centroids of the $k$ clusters of $S$.
Subsequently, we obtain a feasible $k$-clustering of $X$, simply by grouping the data points in the same Voronoi cell together. Then we choose the best one among
these possible results. The Random Sampling algorithm is provided as Algorithm \ref{alg}. Next, we estimate the value for each of the $k$ clusters of $X$.

\begin{algorithm}[htb]
\label{alg}
\caption{Random Sampling for $k$-clustering tasks}
\textbf{Input:} Dataset $X$, integer $k$;\\
\textbf{Output:} $k$-clustering of $X$.\\
\vspace{-0.2cm}
\hrulefill\\
\nl Sample a subset $S$ by $m(\ge k)$ independent draws from $X$ uniformly at random;\\
\nl \For {every $k$-clustering $\{S_i\}_{1\le i\le k}$ of $S$}{
\nl Compute the centroid set $C=\{c(S_i)\}_{1\le i\le k}$;\\
\nl Obtain $\{X_i\}_{1\le i\le k}$, the Voronoi Partition of $X$ w.r.t. the Voronoi Diagram generated by $C$;\\
\nl Compute the value $\sum\limits_{i=1}^k\sum\limits_{x\in X_i}||x-c(X_i)||^2$;\\}
\nl \Return $\{X_i\}_{1\le i\le k}$ with the minimum value.
\end{algorithm}

Let $\{X_i^\prime\}_{1\le i\le k}$ be the output of the Random Sampling algorithm, from which we obtain the corresponding $k$-clustering $\{S_i^\prime\}_{1\le i\le k}$
of the random sampling subset $S^\prime$. Because the centroid of each cluster in $\{X_i^\prime\}_{1\le i\le k}$ defines a Voronoi cell of the space, according to which
we partition $S^\prime$ into $k$-clustering.
Assume w.o.l.g. that $|S^\prime_i|\le |S^\prime_{i+1}|$ for $i=1,2,...,k-1$.
Suppose $\{X_i^*\}_{1\le i\le k}$ is the optimal solution such that $|X^*_i|\le |X^*_{i+1}|$ for $i=1,2,...,k-1$.
Since $S^\prime$ is obtained from $m$ independent draws from $X$, the size of each cluster in $\{S_i^\prime\}_{1\le i\le k}$ is determined by independent Bernoulli trials, and is dependent on
the distribution of $|X_i^*|$ over all $i$. Thus
it must be that $E(|S^\prime_i|)=\frac{m}{n}E(|X^*_i|)$.
We denote the distribution function of $|X_i^*|$ by $p(i):=\frac{|X_i^*|}{n}$ over all $i\in \{1,...,k\}$.
We call $X$ a $\mu$-balanced instance ($0\le \mu \le 1$) if there exists an optimal $k$-clustering for $X$ such that all clusters have size at least $\mu |X|$.
For example, if $p(1)\ge \mu$, then we call $X$ a $\mu$-balanced instance. Recall $X_1^*$ is the smallest cluster in $\{X_i^*\}_{1\le i\le k}$.
We obtain the following lemma.

\begin{lemma} \label{size}
If $X$ is a $(\ln m/m)$-balanced instance, then for any small positive constant $\eta$, it holds with probability at least $1-m^{-\eta^2/2}$ that
$$|S^\prime_i|\ge (1-\eta)mp(i)$$ for all $i=1,...,k$.
\end{lemma}
\begin{proof}
It is obvious that $$E(|S_i^\prime|)=\frac{m}{n}E(|X^*_i|)=mp(i).$$
We now start the proof with $S_1^\prime$, the smallest cluster in expectation.
Consider $m$ rounds of the following Bernoulli trial
\begin{equation}
\left\{
\begin{array}{lr}
1, & {\rm with~ probability}~ p(1);\\
0, & {~~~\rm with~ probability}~ 1-p(1).
\end{array}
\right. \nn
\end{equation}
Let $B_1, B_2,..., B_m$ be the independent random variables of the $m$ trials and let $B=\sum_{i=1}^m B_i$. Obviously $E(B)=mp(1)$ and from the Chernoff Bound we have
$${\rm Pr}[B<(1-\eta)mp(1)]<e^{-\frac{mp(1)\eta^2}{2}}\le e^{-\frac{\ln m\eta^2}{2}}=m^{-\frac{\eta^2}{2}}.$$
Thus, with probability at least $1-m^{-\eta^2/2}$, it follows that $$|S^\prime_1|\ge (1-\eta)mp(1).$$
Similarly for $i=2,...,k$ as $p(i)\ge \ln m/m$ hold for all $i$, complete the proof.
\end{proof}

By combining Lemma \ref{markov} and \ref{size}, we conclude the following estimate for the Random Sampling algorithm.

\begin{theorem}
For any $(\ln m/m)$-balanced instance of a $k$-clustering task, Algorithm 1 returns a feasible solution that it is with probability at least $1-\delta-m^{-\eta^2/2}$ within
a factor of $1+\frac{1}{(1-\eta)\delta\ln m}$ to the optimum.
\end{theorem}
\begin{proof}
Considering the objective value of the output of Algorithm 1, and using the Centroid Lemma, we have
$$\sum\limits_{i=1}^k\sum\limits_{x\in X_i^\prime}||x-c(X_i^\prime)||^2\le \sum\limits_{i=1}^k\sum\limits_{x\in X_i^\prime}||x-c(S_i^\prime)||^2.$$
From line 4 of Algorithm 1, we know that the partition $\{X_i^\prime\}_{1\le i\le k}$ is obtained from the Voronoi Diagram generated by $\{c(S_i^\prime)\}_{1\le i\le k}$. That is to say,
for any $x\in X_i^\prime$ and an arbitrary $j\neq i$, it must be the case that $$||x-c(S_i^\prime)||\le ||x-c(S_j^\prime)||.$$
Summing over all $x$, we obtain
$$\sum\limits_{i=1}^k\sum\limits_{x\in X_i^\prime}||x-c(S_i^\prime)||^2\le \sum\limits_{i=1}^k\sum\limits_{x\in X_i^*}||x-c(S_i^\prime)||^2.$$
The right hand side implies an assignment where an $x$ is assigned to $c(S_i^\prime)$ as long as $x\in X_i^*$ for some $i$.
Considering an $x\in X_i^*$, we do not change its cost of those $x\in X_i^* \cap X^\prime_i$. But we increase the cost of those  $x\in X_i^*\cap X^\prime_j$ for any $j\neq i$.

Applying Lemma \ref{markov} to every cluster in $\{X_i^*\}_{1\le i\le k}$, with probability at least $1-\delta$, it holds that
$$\sum\limits_{i=1}^k\sum_{x\in X_i^*}||x-c(S_i^\prime)||^2\le \sum\limits_{i=1}^k\sum_{x\in X_i^*}(1+\frac{1}{\delta|S_i^\prime|})||x-c(X_i^*)||^2.$$
Combining with Lemma \ref{size}, we obtain with probability at least $(1-\delta)(1-m^{\eta^2/2})\approx 1-\delta-m^{\eta^2/2}$ that
\bea
&&\sum\limits_{i=1}^k\sum_{x\in X_i^*}||x-c(S_i^\prime)||^2\nn\\
&\le& \sum\limits_{i=1}^k\sum_{x\in X_i^*}(1+\frac{1}{(1-\eta)\delta mp(i)})||x-c(X_i^*)||^2\nn\\
&\le& (1+\frac{1}{(1-\eta)\delta\ln m})\sum\limits_{i=1}^k\sum\limits_{x\in X_i^*}||x-c(X_i^*)||^2,\nn
\eea
where the last inequality follows from the assumption that $X$ is a $(\ln m/m)$-balanced instance. Complete the proof.

\end{proof}

\section{Numerical Results}
In this section, we evaluate the performance of
the proposed RS (abbreviation for the Random Sampling algorithm) mainly through the cross comparisons with the widely known KM (abbreviation for the $k$-means method) and KM++ (abbreviation for the $k$-means++ method) on the same datasets.
The environment for experiments is Intel(R) Xeon(R) CPU E5-2620 v4 @ 2.10GHz with 64GB memory. We construct
extensive numerical experiments to analyze different impacts of the proposed algorithm as well as the parameter settings.
Since all algorithms are randomized,
we run RS, KM and KM++ on 100 instances per setting and report the number of instances of each algorithm hitting the minimum objective value.
We mainly design the following experiments due to disparate purposes.

\begin{description}
  \item[1) Effect of $n$:]\

We generate 100 instances of each $n=\{100,200,...,1000\}$ with a standard normal distribution, after which we run simultaneously the RS, KM and KM++ on the same instance
and record which of the three algorithms hits the minimum objective value. We fix $m=n/10$, $k=3$ throughout the experiments and see Figure \ref{exp1} the numerical results.
\begin{figure}[h]
\centering
\begin{tikzpicture}[scale=0.8]
\begin{axis}[
    legend style={legend columns=-1},
    ylabel={\# instances hit the minimum},
    xlabel={value of $n$},
    ymin=0, ymax=100,
    ybar=0pt,
    bar width=3pt,
    symbolic x coords={100,200,300,400,500,600,700,800,900,1000},
    xtick={100,200,300,400,500,600,700,800,900,1000},
    x tick label style={rotate=45,anchor=east},
    ]
\addplot coordinates {
(100,81)
(200,78)
(300,71)
(400,68)
(500,55)
(600,51)
(700,42)
(800,43)
(900,32)
(1000,26)
};

\addplot coordinates {
(100,86)
(200,80)
(300,66)
(400,66)
(500,53)
(600,46)
(700,39)
(800,40)
(900,36)
(1000,35)
};

\addplot coordinates {
(100,14)
(200,13)
(300,19)
(400,24)
(500,27)
(600,31)
(700,46)
(800,43)
(900,44)
(1000,48)
};
\legend{KM,KM++,RS}
\end{axis}
\end{tikzpicture}
\caption{Effect of the size of the dataset $n$}
\label{exp1}
\end{figure}
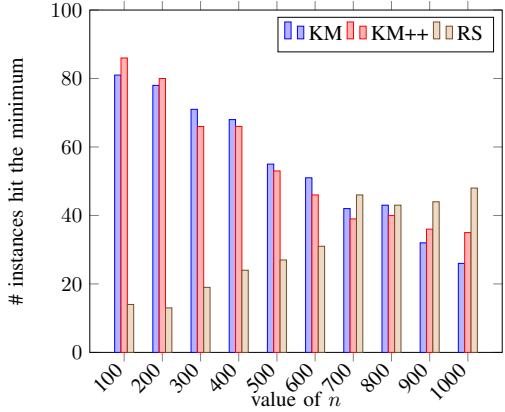

The RS performs not so good as KM or KM++ at the beginning because the sampling set is too small to represent the entire dataset.
Taking $n=100$ as an example, a 10-sized sampling set is probably not a good estimate for the original 100-sized dataset. However,
when $n$ increases to 700, a 70-sized sampling set seems good enough for RS to be competitive with KM and KM++. With the rise of $n$,
RS performs increasingly better and tends to outperform both the KM and KM++. Note that fixing $n=100$ for example,
the total number of instances that any of the three algorithms hitting the minimum exceeds 100. This is because for smaller instances,
it is more likely that not only one algorithm is hitting the minimum, and in this case we count all of them once in Figure \ref{exp1}.

  \item[2) Effect of $k$:]\

We generate 100 instances with a standard normal distribution, after which we run simultaneously the RS, KM and KM++ on the same instance
 for different $k$-clustering tasks with each $k=\{2,3,...,8\}$,
and record which of the three algorithms hits the minimum objective value. We fix $n=100$, $m=50$ throughout the experiments and see Figure \ref{exp2} the numerical results.
\begin{figure}[h]
\centering
\begin{tikzpicture}[scale=0.8]
\begin{axis}[
    legend style={legend columns=-1},
    ylabel={\# instances hit the minimum},
    xlabel={vaue of $k$},
    ymin=0, ymax=100,
    ybar=0pt,
    bar width=4pt,
    symbolic x coords={2,3,4,5,6,7,8},
    xtick={2,3,4,5,6,7,8},
    ]
\addplot coordinates {
(2,71)
(3,100)
(4,72)
(5,51)
(6,42)
(7,42)
(8,24)
};

\addplot coordinates {
(2,80)
(3,99)
(4,67)
(5,57)
(6,48)
(7,27)
(8,43)
};

\addplot coordinates {
(2,98)
(3,65)
(4,40)
(5,14)
(6,35)
(7,36)
(8,35)
};
\legend{KM,KM++,RS}
\end{axis}
\end{tikzpicture}
\caption{Effect of the number of clusters $k$}
\label{exp2}
\end{figure}
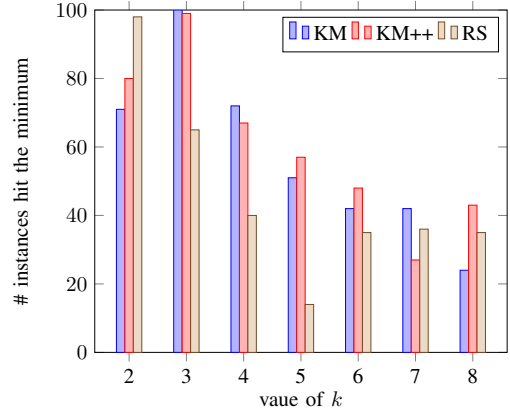

As shown, the RS reaches the best performance in 2-clustering and worst performance in 5-clustering.
Overall, it is competitive with KM and KM++ with these settings.

  \item[3) Effect of $m$:]\

We evaluate the performance of our algorithm on real-world dataset.
The Cloud dataset consists of 1024 points and represents the 1st cloud cover database available from the UC-Irvine Machine Learning Repository.
We run simultaneously the KM and KM++ on the Cloud dataset, along with the RS with each sampling size $m=\{25,50,75,...,200\}$.
Since there is only one instance here, we run 100 rounds of each algorithm per setting and report the one hitting the minimum objective value.
Note that $n=1024$ and we fix $k=3$ throughout the experiments and see Figure \ref{exp3} the numerical results.

\begin{figure}[h]
\centering
\begin{tikzpicture}[scale=0.8]

\begin{axis}[
    legend style={legend columns=-1},
     legend pos=north west,
    ylabel={\# rounds hit the minimum},
    xlabel={value of $m$},
    ymin=0, ymax=100,
    ybar=0pt,
    bar width=4pt,
    symbolic x coords={25,50,75,100,125,150,175,200},
    xtick=data,
    ]
\addplot coordinates {
(25,80)
(50,62)
(75,52)
(100,38)
(125,17)
(150,14)
(175,9)
(200,3)
};

\addplot coordinates {
(25,69)
(50,65)
(75,33)
(100,27)
(125,17)
(150,11)
(175,9)
(200,5)
};

\addplot coordinates {
(25,2)
(50,17)
(75,36)
(100,56)
(125,74)
(150,81)
(175,87)
(200,92)
};
\legend{KM,KM++,RS}
\end{axis}
\end{tikzpicture}
\caption{Effect of the size of the sampling set $m$}
\label{exp3}
\end{figure}
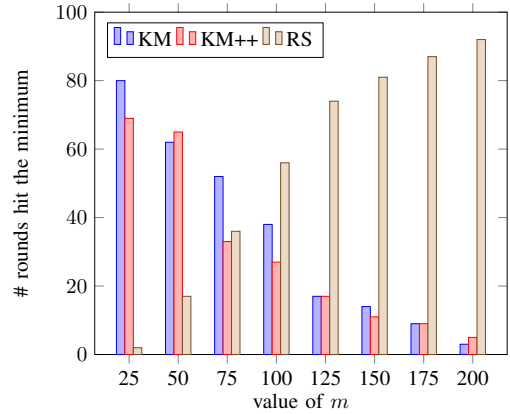

As predicted, the RS performs increasingly better when the sampling size gets large.
But it is quite surprising that when $m=75$ (about only 7\% of the Cloud dataset), the RS performs as good as KM++.
When $m$ is higher than 100 (about 10\% of the Cloud dataset), the RS outperforms any one of the KM and KM++.
If $m$ reaches 150 (about 15\% of the Cloud dataset) or higher, the RS wins in at least 80\% rounds of the clustering tasks.
\end{description}

\section{Extension to Balanced $k$-Clustering}
An additional important feature of the proposed Random Sampling algorithm is the extension to handle the balanced variance-based $k$-clustering tasks, for which the $k$-means method
and the $k$-means++ method can not deal with. Both upper bound and lower bound constraints are considered, which means a feasible balanced $k$-clustering has a global lower bound $l$ and an upper bound $u$
for the cluster sizes. We assume w.o.l.g. that $l$ and $u$ are positive integers. The main idea is a minimum-cost flow subroutine embedded into the Random Sampling algorithm.

To start, we introduce the well-known minimum-cost flow problem. Given a directed graph $G=(V,E)$, every edge $e\in E$ has a weight $c(e)$ representing its cost of sending
a unit of flow.  Also, every $e\in E$ is equipped with a bandwidth constraint. Only those flows within a maximum flow value of $upper(e)$ and minimum value of $lower(e)$ can pass
through edge $e$ for each $e\in E$, where $upper(e)$ and $lower(e)$ denote the upper bound and the lower bound for the bandwidth of $e$ respectively.
Every node $v\in V$ has a demand $d(v)$, defined as the total outflow minus total inflow.
Thus a negative demand represents a need for flow and a positive one represents a supply.

A flow in $G$ is defined as a function from $V$ to $\mathbb{R}^+$.
A feasible flow carrying $f$ amount of flow in the graph requires a source $s$ and a sink $t$ with $d(s)=f$ and $d(t)=-f$.
Every node $v\in V\setminus \{s,t\}$ must have $d(v)=0$, which means it is either an intermediate node or an idle node.
The cost of flow $f$ is defined as $c(f)=\sum_{e\in E}f(e)\cdot c(e)$, where $f(\cdot): V\rightarrow \mathbb{R}^+$ is the corresponding function of flow $f$.
The minimum--cost flow is the optimization problem to find a
cheapest way (i.e. with the minimum cost) of sending a certain amount of flow through graph $G$.

To deal with the capacity constraints, we herein propose a Random Sampling based randomized algorithm embedding in the minimum--cost flow subroutine.
Obviously, the Voronoi Diagram generated by the centroids of the $k$-clustering of the sampling set $S$ does not guarantee a feasible Voronoi Partition of $X$ satisfying the capacity constraints.
Assume that we are given a $k$-clustering of $S$ and we look for a feasible balanced $k$-clustering of $X$.

Consider the following instance of the minimum--cost flow problem.
Let $V$ be  $X\cup C\cup \{s,t\}$, where $C$ consists of  the centroids $\{c(C_i)\}_{1\le i\le k}$ obtained from the given $k$-clustering of $S$, and $s$ and $t$ are the dummy source and sink nodes respectively.
Let $E$ be $E_1\cup E_2\cup E_3$, where $E_1$ are the directed edges $(s,i)$ from $s$ to each $i\in X$,
$E_2$ are the edges $(i,j)$ from each $i\in X$ to $j\in C$, and $E_3$ are the edges $(j,t)$ from each $j\in C$ to $t$. Every edge in $E_1\cup E_2$ has bandwidth interval $[0,1]$ while $E_3$ has
$[l,u]$. Edges in $E_1\cup E_3$ are unweighted and edge $(i,j)\in E_2$ has weight $||i-j||^2$ for each $i\in X$ and $j\in C$. See Figure \ref{mcf} as a description.

\begin{figure}[htb]
  \begin{center}
  \includegraphics[width=8.7cm]{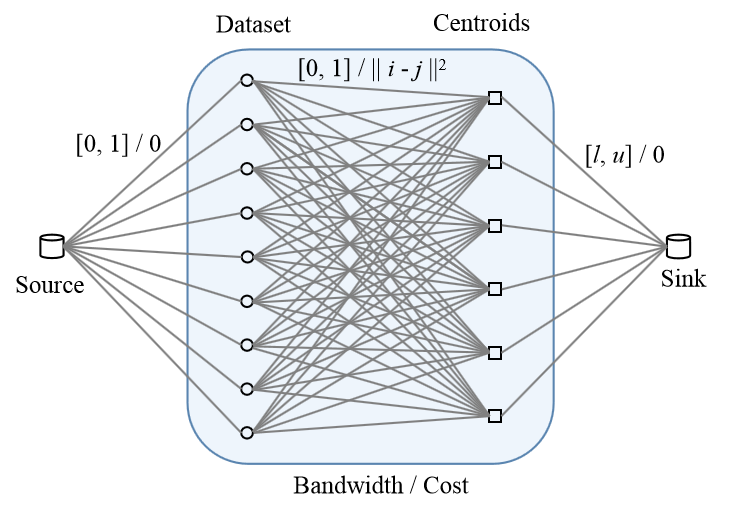}
\\\end{center}
\vspace{-0.4cm}
  \caption{A minimum-cost flow instance}
  \label{mcf}
\end{figure}

As shown in the figure, the bandwidth intervals and the weights/costs are labeled on the edges. All the edges are oriented from the source to the sink and we simply omit the direction labels.
Inside the shadowed box is a complete bipartite graph, also known as a biclique, consisting of vertices $X\cup C$ and edges $E_2$.
Consider a flow $f$ that carrying $n$ ($n=|X|$) amount of flow from the source to the sink in $G$ and suppose that function $f: E\mapsto\mathbb{R}^+$ reflects such a flow. Recall
that $d(v)=\sum_{e\in \delta^+(v)}f(e)-\sum_{e\in \delta^-(v)}f(e)$, where $\delta^+(v)$ denotes the edges leading away from node $v$ and $\delta^-(v)$ denotes the edges leading into $v$.
Then the follows must hold.

\begin{itemize}
  \item Flow conservation:
    \bea
d(v)=\left\{
\begin{array}{ll}
n, &v=s;\\
-n, &v=t;\\
0,& \forall~v\neq s, t.\nn
\end{array} \right.
\eea

  \item Bandwidth constraints:
   \bea
\left\{
\begin{array}{ll}
0\le f(e)\le 1, &\forall e\in E_1;\\
0\le f(e)\le 1, &\forall e\in E_2;\\
l\le f(e)\le u, &\forall e\in E_3. \nn
\end{array}
\right.
\eea
\end{itemize}


Then minimum--cost flow problem aims to find a function $f: E\mapsto\mathbb{R}^+$ satisfying both the flow conservation and the bandwidth constraints so as to minimize its cost, i.e., $\sum_{e\in E}f(e)\cdot c(e)$.
An important property of the minimum-cost flow problem is that basic feasible solutions are integer-valued if capacity constraints and quantity of
flow produced at each node are integer-valued, as captured by the following lemma.
\begin{lemma} \cite{amo}
If the objective value of the minimum-cost flow is bounded from below on the feasible region, the problem has a feasible solution, and
if capacity constraints and quantity of flow are all integral, then the problem has at least one integral optimal solution.
\end{lemma}

The integral solution can be computed efficiently by the Cycle Canceling algorithms, Successive Shortest Path algorithms, Out-of-Kilter algorithms and Linear Programming based algorithms.
These algorithms can be found in many textbooks. See for example \cite{amo}. We take any one of these algorithms as the MCF (Minimum-Cost Flow) subroutine in our algorithm. We show the following theorem.

\begin{theorem}
The integral optimal solution to the above minimum-cost flow instance provides an optimal assignment from $X$ to $C$ for balanced clustering tasks.
\end{theorem}
\begin{proof}
We only need to prove that any feasible assignment from the dataset $X$ to the given centroid $C$ can be represented by a feasible integral flow to the aforementioned minimum-cost
flow instance, and vice versa.

Let $\sigma: X\mapsto C$ be a feasible assignment from $X$ to $C$.
Consider the following flow $f: E\mapsto\mathbb{R}^+$.
\bea
f(e)=
\left\{
\begin{array}{ll}
1, &\forall e\in E_1,\\
1, &\forall e\in E_2~{\rm and}~\sigma(e_o)=e_d,\\
0, &\forall e\in E_2~{\rm and}~\sigma(e_o)\neq e_d,\\
\sum_{e^\prime:e^\prime_d=e_o}f(e^\prime), &\forall e\in E_3, \nn
\end{array}
\right.
\eea
where we denote the origin and destination of edge $e$ by $e_o$ and $e_d$ respectively.
Note that the quantity of $f$ is $n$. Obviously, $f$ satisfies the flow conversation and every edge in $E_1$ and $E_2$ obeys the bandwidth constraints. For $e\in E_3$, from the construction we have
$f(e)=\sum_{e^\prime:e^\prime_d=e_o}f(e^\prime)=|\sigma^{-1}(e_o)|$. Since $\sigma$ is feasible, then it must hold for every $j\in C$ that $l\le |\sigma^{-1}(j)|\le u$, which implies the feasibility of the
bandwidth constraints for $E_3$.

On the other hand, given an integral feasible flow $f$, the corresponding assignment must be feasible, i.e., satisfying the size constraints. Note that a feasible flow with quantity $n$ in the
above instance must have all $f(e)=1$ for every $e\in E_1$.
Consider the following assignment $\sigma$: For any  $i\in X, j\in C$, $\sigma(i)=j$ if and only if an edge with $e_o=i$ and $e_d=j$ is such that $f(e)=1$. The defined assignment must be feasible
because $|\sigma^{-1}(j)|=\sum_{e\in \delta^-(j)}f(e)=\sum_{e\in \delta^+(j)}f(e)$ holds for any $j\in C$. Then from the feasibility of flow $f$ we know that $l\le \sum_{e\in \delta^+(j)}f(e)\le u$.

It is obvious that the cost of a feasible assignment and the cost of its corresponding flow are exactly the same. Because
\bea
\sum_{e\in E}c(e)f(e)&=&\sum_{e\in E_2}c(e)f(e)\nn\\
 &=&\sum_{e\in E_2:f(e)=1}||e_o-e_d||^2 \nn\\
&=& \sum_{i\in X}\sum_{j=\sigma(i)}||i-j||^2 \nn\\
&=&\sum_{x\in X}||x-\sigma(x)||^2\nn\\
&=& \sum_{i=1}^k\sum_{x\in X_i}||x-\sigma(x)||^2,\nn
\eea
where the first equality is derived from the construction and the last equality holds for any feasible partition of $X$, which we assume without loss of generality is $\{X_i\}_{1\le i\le k}$.
Implies the lemma.
\end{proof}
\begin{algorithm}[h]
\label{alg2}
\caption{Random Sampling for balanced $k$-clustering tasks}
\textbf{Input:} Dataset $X$, integer $k$;\\
\textbf{Output:} $k$-clustering of $X$.\\
\vspace{-0.2cm}
\hrulefill\\
\nl Sample a subset $S$ by $m(\ge k)$ independent draws from $X$ uniformly at random;\\
\nl \For {every $k$-clustering $\{S_i\}_{1\le i\le k}$ of $S$}{
\nl Compute the centroid set $C=\{c(S_i)\}_{1\le i\le k}$;\\
\nl Obtain $\{X_i\}_{1\le i\le k}$ by the MCF subroutine;\\
\nl Compute the value $\sum\limits_{i=1}^k\sum\limits_{x\in X_i}||x-c(X_i)||^2$;\\}
\nl \Return $\{X_i\}_{1\le i\le k}$ with the minimum value.
\end{algorithm}

Based on the above, we conclude that a MCF subroutine embedded in the Random Sampling algorithm guarantees a valid solution for the balanced $k$-clustering problem.
The pseudocode is provided as Algorithm \ref{alg2}.

\section{Discussion}
We are incredibly well informed yet we know incredibly little, and this is what is happening in the clustering tasks.
Our work implies that we do not need so much information of dataset when doing clustering.
From the experiments, roughly speaking, to obtain a competitive clustering result compared with the $k$-means method and $k$-means++ method,
we only need about 7\% information of the dataset. For the rest of the 93\% data, we immediately make decisions for them with only $O(k)$ additional computations.
Note that the resources consumed in the algorithm are dominated by the brute force search for the $k$-clustering of
the sampling set. If we have up to 15\% information of the dataset, then with high probability, our algorithm outperforms both the $k$-means method and $k$-means++ method
in terms of the quality of clustering. The above statements hold only when 1) The dataset is independent and identically distributed; 2) The sampling set is picked uniformly
at random from the original dataset; 3) The most important, the dataset is large enough (experimentally 500 data points or above suffice).
At a cost, the proposed algorithm has a high complexity with respect to $k$, but fortunately not sensitive to the size of the dataset or the size of the sampling set.

We believe that the Random Sampling idea as well as the framework of the analysis has the potential to deal with
incomplete dataset and online clustering tasks.

\bibliographystyle{plain}
\bibliography{mybibfile}

\begin{thebibliography}{10}

\bibitem{amo}
Ravindra~K. Ahuja, Thomas~L. Magnanti, and James~B. Orlin.
\newblock {\em Network flows - theory, algorithms and applications}.
\newblock Prentice Hall, 1993.

\bibitem{km++}
David Arthur and Sergei Vassilvitskii.
\newblock k-means++: The advantages of careful seeding.
\newblock In {\em ACM-SIAM Symposium on Discrete Algorithms (SODA)}, pages
  1027--1035, 2007.

\bibitem{hardness}
Pranjal Awasthi, Moses Charikar, Ravishankar Krishnaswamy, and Ali~Kemal Sinop.
\newblock The hardness of approximation of euclidean k-means.
\newblock In {\em International Symposium on Computational Geometry (SoCG)},
  pages 754--767, 2015.

\bibitem{b2016}
Olivier Bachem, Mario Lucic, S~Hamed Hassani, and Andreas Krause.
\newblock Approximate k-means++ in sublinear time.
\newblock In {\em AAAI Conference on Artificial Intelligence (AAAI)}, pages
  1459--1467, 2016.

\bibitem{km||}
Bahman Bahmani, Benjamin Moseley, Andrea Vattani, Ravi Kumar, and Sergei
  Vassilvitskii.
\newblock Scalable k-means++.
\newblock In {\em Very Large Data Bases (VLDB)}, pages 622--633, 2012.

\bibitem{arxiv2020}
Davin Choo, Christoph Grunau, Julian Portmann, and V{\'a}clav Rozho{\v{n}}.
\newblock k-means++: few more steps yield constant approximation.
\newblock {\em arXiv preprint arXiv:2002.07784}, 2020.

\bibitem{QPTAS}
Vincent Cohen-Addad.
\newblock Approximation schemes for capacitated clustering in doubling metrics.
\newblock In {\em ACM-SIAM Symposium on Discrete Algorithms (SODA)}, pages
  2241--2259, 2020.

\bibitem{ICALP}
Vincent Cohen-Addad and Jason Li.
\newblock On the fixed-parameter tractability of capacitated clustering.
\newblock In {\em International Colloquium on Automata, Languages, and
  Programming (ICALP)}, pages 1--14, 2019.

\bibitem{socg94}
Mary Inaba, Naoki Katoh, and Hiroshi Imai.
\newblock Applications of weighted voronoi diagrams and randomization to
  variance-based k-clustering.
\newblock In {\em International Symposium on Computational Geometry (SoCG)},
  pages 332--339, 1994.

\bibitem{jv}
Kamal Jain and Vijay~V Vazirani.
\newblock Approximation algorithms for metric facility location and k-median
  problems using the primal-dual schema and lagrangian relaxation.
\newblock {\em Journal of the ACM}, 48(2):274--296, 2001.

\bibitem{betterkm}
Silvio Lattanzi and Christian Sohler.
\newblock A better k-means++ algorithm via local search.
\newblock In {\em International Conference on Machine Learning (ICML)}, pages
  3662--3671, 2019.

\bibitem{AAAI18}
Zhihui Li, Feiping Nie, Xiaojun Chang, Zhigang Ma, and Yi~Yang.
\newblock Balanced clustering via exclusive lasso: A pragmatic approach.
\newblock In {\em AAAI Conference on Artificial Intelligence (AAAI)}, pages
  3596--3603, 2018.

\bibitem{xmy}
Weibo Lin, Zhu He, and Mingyu Xiao.
\newblock Balanced clustering: A uniform model and fast algorithm.
\newblock In {\em International Joint Conference on Artificial Intelligence
  (IJCAI)}, pages 2987--2993, 2019.

\bibitem{AAAI17}
Hanyang Liu, Junwei Han, Feiping Nie, and Xuelong Li.
\newblock Balanced clustering with least square regression.
\newblock In {\em AAAI Conference on Artificial Intelligence (AAAI)}, pages
  2231--2237, 2017.

\bibitem{l1982}
Stuart Lloyd.
\newblock Least squares quantization in pcm.
\newblock {\em IEEE transactions on information theory}, 28(2):129--137, 1982.

\bibitem{wu2008}
Xindong Wu, Vipin Kumar, J~Ross Quinlan, Joydeep Ghosh, Qiang Yang, Hiroshi
  Motoda, Geoffrey~J McLachlan, Angus Ng, Bing Liu, S~Yu Philip, et~al.
\newblock Top 10 algorithms in data mining.
\newblock {\em Knowledge and information systems}, 14(1):1--37, 2008.

\bibitem{xu}
Yicheng Xu, Rolf~H M{\"o}hring, Dachuan Xu, Yong Zhang, and Yifei Zou.
\newblock A constant fpt approximation algorithm for hard-capacitated k-means.
\newblock {\em Optimization and Engineering}, pages 1--14, 2020.

\end{thebibliography}

\end{document}